\documentclass{article} 

\usepackage{geometry}
\usepackage[utf8]{inputenc}
\usepackage{babel,csquotes,xpatch}
\usepackage[numbers,sort]{natbib}

\usepackage{etoolbox}
\makeatletter
\providecommand{\institute}[1]{
  \apptocmd{\@author}{\end{tabular}
    \bigskip

    \begin{tabular}[t]{c}
    #1}{}{}
} 
\makeatother

\usepackage{latexsym}
\usepackage{amssymb}
\usepackage{amsmath}
\usepackage{amsthm}
\usepackage{booktabs}
\usepackage{enumitem}
\usepackage{graphicx}
\usepackage{color}

\newtheorem{theorem}{Theorem}
\newtheorem{lemma}[theorem]{Lemma}
\newtheorem{corollary}[theorem]{Corollary}
\newtheorem{proposition}[theorem]{Proposition}

\newtheorem{definition}[theorem]{Definition}
\newtheorem{example}[theorem]{Example}
\newtheorem{observation}[theorem]{Observation}

\usepackage{booktabs}
\usepackage{algorithm}
\usepackage{algorithmic}
\usepackage{tikz}
\usepackage{macros}
\usepackage{multicol}
\newcommand{\CAF}{CF}
\newcommand{\CF}{G}
\newcommand{\tw}[1]{\mathsf{tw}(#1)}
\newcommand{\primal}[1]{\mathcal{G}_{#1}}

\DeclareMathOperator{\adef}{def}

\newcommand{\forg}{\mathsf{last}}

\newcommand{\conf}{\mathsf{conf}}

\newcommand{\adm}{\mathsf{adm}}
\newcommand{\comp}{\mathsf{comp}}
\newcommand{\pref}{\mathsf{pref}}
\newcommand{\stab}{\mathsf{stab}}
\newcommand{\tight}{\mathsf{tight}}
\newcommand{\normal}{\mathsf{normal}}
\newcommand{\prop}{\mathsf{prop}}
\newcommand{\props}{\mathsf{propositional}}
\newcommand{\simple}{\mathsf{simple}}
\newcommand{\disj}{\mathsf{disj}}
\newcommand{\disjs}{\mathsf{disjunctive}}
\newcommand{\stag}{\mathsf{stag}}
\newcommand{\semi}{\mathsf{semiSt}}
\newcommand{\ALL}{\ensuremath{\mathbb{S}}\xspace}
\newcommand{\sem}{\ensuremath{\sigma}\xspace}

\newcommand{\TQSAT}{\textsc{2-}\QSAT\xspace}

\newcommand{\QSAT}{\textsc{QSat}\xspace}

\newcommand{\BSAT}{\textsc{Sat}\xspace}

\newcommand{\Stab}{\ensuremath{\mathsf{stab}}\xspace}

\DeclareMathOperator{\width}{\mathsf{width}}
\newcommand{\Card}[1]{\left|#1\right|}

\newcommand{\TTT}{\ensuremath{\mathcal{T}}}%

\newcommand{\at}{\var} %

\newcommand{\eqdef}{\coloneqq}

\newcommand{\SB}{\ensuremath{\{\,}}
\newcommand{\SE}{\ensuremath{\,\}}}
\newcommand{\SM}{\ensuremath{\;|\;}}
\newcommand{\matr}{\mathrm{matr}}
\newcommand{\var}{\mathrm{var}}

\newcommand{\class}[1]{\textbf{#1}}
\newcommand{\Simple}[0]{\class{Simple}\xspace}
\newcommand{\Normal}[0]{\class{Normal}\xspace}
\newcommand{\Tight}[0]{\class{Tight}\xspace}
\newcommand{\Prop}[0]{\class{Prop.}\xspace}
\newcommand{\Disj}[0]{\class{Disj}\xspace}

\DeclareMathOperator{\children}{children}
\DeclareMathOperator{\tower}{\ensuremath{\mathsf{tower}}}
\DeclareMathOperator{\poly}{\ensuremath{{poly}}}
\def\hy{\hbox{-}\nobreak\hskip0pt}

\newcommand{\cred}{\mathrm{cred}}
\newcommand{\cons}{\mathsf{cons}}

\newcommand{\rootOf}{\mathrm{rt}}

\newcommand{\brm}[1]{\mathbf{\mathrm{#1}}}

\newcommand{\citey}[1]{\cite{#1}}
\newcommand{\RAF}[0]{RCF\xspace}
\newcommand{\camera}[1]{}

\title{Rejection in Abstract Argumentation: \\Harder Than Acceptance?}
\author{Johannes K. Fichte$^{1}$ \and Markus Hecher$^2$ \and Yasir Mahmood$^3$ \and Arne Meier$^4$}

\institute{$^1$: Link\"oping University, $^2$: Massachusetts Institute of Technology\\ $^3$: Universit\"at Paderborn, $^4$: Leibniz Universit\"at Hannover}

\begin{document}
\maketitle
\begin{abstract}
    Abstract argumentation is a popular toolkit for modeling,
    evaluating, and comparing arguments.  Relationships between
    arguments are specified in argumentation frameworks (AFs), and
    conditions are placed on sets (extensions) of arguments that allow
    AFs to be evaluated.  For more expressiveness, AFs are augmented
    with \emph{acceptance conditions} 
    on directly interacting arguments or a constraint on the admissible
    sets of arguments, resulting in dialectic frameworks or constrained
    argumentation frameworks.
    In this paper, we consider flexible conditions for \emph{rejecting}
    an argument from an extension, which we call rejection conditions
    (RCs).  On the technical level, we associate each argument with a
    specific logic program.
    We analyze the resulting complexity, including the structural
    parameter treewidth.
    Rejection AFs are highly expressive, giving rise to natural problems
    on higher levels of the polynomial hierarchy.
\end{abstract}

\section{Introduction}
Argumentation is a popular area of AI~\cite{AmgoudPrade09a,RagoCocarascuToni18a}.
An important computational framework, also known as abstract argumentation is Dung's argumentation framework (AF)~\cite{Dung95a}.
AFs are widely used for modeling, evaluating, and comparing arguments.
The semantics is based on sets of arguments that satisfy certain conditions such as being admissible, stable, or preferred. 
Such sets of arguments are then called \emph{extensions} of a framework.

Over time, various weaknesses of the somewhat simple modeling of arguments via AFs have been dictated, followed by proposed extensions, thereby providing another layer of “acceptance” conditions for arguments. 
For instance, Constrained Argumentation Frameworks (CAFs) \cite{Coste-MarquisDevredMarquis06a} or Abstract Dialectical Frameworks (ADFs)~\cite{BrewkaWoltran10}. 
\begin{table*}[t]
	\resizebox{\linewidth}{!}{
          \begin{tabular}{ccccc}
            \toprule
            \backslashbox{\textbf{RC}}{\textbf{Problem}}  & 
            $\cons_{\sigma}$/$\cred_{\tau}$ & 
            TW: $\cons_{\sigma}$/$\cred_{\tau}$ & 
            $\cred_{\semi}$/$\cred_{\stag}$ & 
            TW: $\cred_{\semi}$/$\cred_{\stag}$ \\
            \midrule
            $\Simple$ & 
            $\NP$ {\footnotesize(\ref{thm:hard}, \ref{thm:membership}/\ref{cor:cred})} & 
            $\tower(1, \Theta(k))$ {\footnotesize(\ref{thm:simple}, \ref{thm:tw-lb})} & 
            $\SigmaP{2}$ {\footnotesize(\ref{thm:credsemistab})} & 
            $\tower(2, \Theta(k))$ {\footnotesize(\ref{thm:disj-tw-ub}, \ref{thm:tw-lb})} \\
            $\Prop$ / $\Tight$ & 
            $\SigmaP{2}$ {\footnotesize(\ref{thm:hard}, \ref{thm:membership}/\ref{cor:cred})} & 
            $\tower(2, \Theta(k))$ {\footnotesize(\ref{thm:prop}/\ref{thm:tight}, \ref{thm:tw-lb})} & 
            $\SigmaP{3}$ {\footnotesize(\ref{thm:credsemistab})} & 
            $\tower(3, \Theta(k))$ {\footnotesize(\ref{thm:disj-tw-ub}, \ref{thm:tw-lb})} \\
            $\Disj$ & 
            $\SigmaP{3}$ {\footnotesize(\ref{thm:hard}, \ref{thm:membership}/\ref{cor:cred})} & 
            $\tower(3, \Theta(k))$ {\footnotesize(\ref{thm:disj-tw}, \ref{thm:tw-lb})} & 
            $\SigmaP{4}$ {\footnotesize(\ref{thm:credsemistab})} & 
            $\tower(4, \Theta(k))$ {\footnotesize(\ref{thm:disj-tw-ub}, \ref{thm:tw-lb})} \\
            \bottomrule
          \end{tabular}}
	\caption{%
          Each row shows the results for a particular class of programs, denoted by the first column. 
          The 2\textsuperscript{nd} column shows the complexity to deciding whether a given \RAF $\CF$ has an extension under semantics $\sigma\in \{\conf,\adm, \comp, \stab, \semi, \stag\}$ or the credulous reasoning problem under semantics $\tau \in \{\conf,\adm, \comp, \stab\}$.
          The 3\textsuperscript{rd} column states the tight bound factor in $k=\tw{\primal{\CF}}$ with respect to the extension existence problem for~$\sigma$.
          The 4\textsuperscript{th} column gives the complexity for credulous reasoning and semantics $\semi$ and $\stag$, followed by its tight bound factor in~$k$ (5\textsuperscript{th} column). 
          All stated results are complete with respect to the respective class of semantics. 
          Numbers in brackets point to our results below. %
      }
\label{tbl:overviewresults}
\end{table*}
But what happens, if we are interested in rejecting conditions?
Some semantics implicitly define the concept of rejection by
(subset\hy)maximality conditions,~e.g., preferred or stage semantics.
However, real-life is more complicated and goes beyond simple
maximization (see Example~\ref{ex:CCF-complex}).
In this paper, we extend argumentation frameworks by flexible concepts
of rejection.
Since in knowledge representation (KR) global constraints are
oftentimes expressed by means of propositional logic, we pursue our
line of work in this direction.
Naturally, we attach logical formulas to arguments, which is common in
argumentation~\cite{Coste-MarquisDevredMarquis06a, BrewkaWoltran10,
  HeyninckThimmKern-Isberner22a}.
The models of these formulas form \emph{counterclaims} to 
arguments being accepted.
As a result the presence of such
counterclaims forms the concept of rejection in AFs, which we call
\emph{rejection framework (\RAF)}.
Interestingly, under reasonable complexity-theoretic assumptions, it turns out that it is harder to express rejection
than acceptance. 
Indeed, compared to other approaches to model acceptance  (e.g., CAFs), the concept of rejection increases computational
hardness by one level already when we allow auxiliary variables.
To express even harder rejection constraints, we study the case where
CCs can be represented as the models (answer sets) of a logic
program~\cite{GebserKaminskiKaufmannSchaub12}, since ASP is commonly
used for argumentation~\cite{GagglEtAl15}.
Specifically, we associate each argument~$a$ with a logical
program as its \emph{rejection condition (RC)}.
When computing extensions, these RCs must be altogether
invalidated and thus no longer counter-attack their arguments. 
This
results in a compact and natural representation of rejection in AFs.
Further, we depict via an example  (Ex.~\ref{ex:CCF-complex}) a scenario where such rejection conditions can not be modeled in AFs without losing the extensions. 

\newcommand{\ns}{\brm{noS}}
\newcommand{\tr}{\brm{T}}
\newcommand{\pr}{\brm{P}}
\newcommand{\pwr}{\brm{W}}
\newcommand{\dl}{\brm{D}}
\newcommand{\nodl}{\brm{noDl}}
\newcommand{\tea}{\brm{Te}}
\newcommand{\re}{\brm{Re}}

\paragraph{Contributions.} 
We complete a previously incomplete picture as rejection was implicitly studied but rarely made explicit.
We naturally extend existing frameworks rather than introducing a new framework. 
Interestingly, this yields a very general understanding of rejection and acceptance in argumentation.

In summary, our main contributions are:
\begin{enumerate}
\item We study flexible rejecting conditions in abstract
  argumentation, which provide more fine-grained conditions to
  commonly studied rejection such as maximality.

\item We investigate whether rejection is harder than its counterpart
  of acceptance conditions. Therefore, we research the complexity of
  decision and reasoning problems, including when restricting
  rejection formulas. We establish detailed complexity results, which we
  survey in Table~\ref{tbl:overviewresults}. Indeed, rejection
  conditions turn out to be harder than accepting conditions.

\item We include the structural parameter treewidth into our
  complexity theoretical considerations and establish matching runtime
  upper and lower bounds.
\end{enumerate}

\paragraph{Related Works.}
\emph{Constrained Argumentation Frameworks (CAFs)} extend AFs with accepting condi\-tions~\cite{Coste-MarquisDevredMarquis06a}, in more detail, a single formula that must be satisfied by the extension. 
The computational complexity remains the same.
For work on CAFs,
see~\cite{DBLP:conf/argmas/AmgoudDL08,DBLP:conf/ipmu/SedkiY16,DBLP:conf/comma/DevredDLN10}.
\emph{Abstract Dialectical Frameworks (ADFs)}~\cite{BrewkaWoltran10}
extend AFs with propositional formulas that state conditions
on arguments in the direct neighborhood. The acceptance condition of
an argument depends on its parent arguments in the attack graph.
There are also versions of ADFs that implement
three-valued~\cite{BrewkaStrassEllmauthaler13a} and four-valued
logics~\cite{HeyninckThimmKern-Isberner22a}.
\emph{Claim-augmented AFs}~\cite{DvorakGresslerRapberger21a} attach
claims to arguments, where arguments in the extensions imply their
respective claims. 
\emph{Assumption-based AFs (ABAs)} take an inference system, such as
logic programs, to specify attacks among
arguments~\cite{BondarenkoDungKowalski1997,DungKowalskiToni06,DungMancarellaToni2007}.
The semantics for ABAs are defined in the usual way, their complexity is well
understood~\cite[Tab.\ 9]{DvorakD17}, and relationship between ABA and logic programming
is well investigated~\cite{SaAlcantara21a}.
Answer Set Programming~\cite{GebserKaminskiKaufmannSchaub12} is a
logic programming modeling and problem solving paradigm where
questions are encoded by atoms combined into rules and constraints
which form a logic program. Preferences allow for stating dedicated
acceptance conditions~\cite{BrewkaDelgrandeRomero15a}.  Answer Set
Programming and their respective solvers have commonly been used for
finding extensions of
AFs~\cite{OsorioEtAl05,GagglEtAl15,CaminadaSchulz17,AlfanoEtAl21}.

\emph{Logic-based argumentation}~\cite{DBLP:journals/ai/BesnardH01} is a different branch of argumentation where one uses logic to construct arguments and model their relationships.
RCFs lets one directly assign rejection conditions to arguments in an AF. These conditions are somewhat independent of each other as they do not need to relate to each other or to arguments in a formal deductive way, which differentiates RCFs and ABA.

\section{Preliminaries}
We assume familiarity with complexity~\cite{Pippenger97}, graph theory~\cite{BondyMurty08}, and
logic~\cite{BiereEtAl21}.
For a set~$X$, we denote $X'\eqdef\{\,x'\mid x\in X\,\}$.

\newcommand{\por}{\ensuremath{\vee}}
\newcommand{\unsim}{\mathord{\sim}}
\newcommand{\pnot}{\ensuremath{\unsim}}

\paragraph{Abstract Argumentation.}
We use Dung's AF~\cite{Dung95a} and consider only non-empty and finite sets of arguments~$A$.
An \emph{(argumentation) framework~(AF)} is a directed graph~$F=(A, R)$ where $A$ is a set of arguments and $R \subseteq A\times A$  a pair of arguments
representing direct attacks of arguments.
An argument~$s \in S$, is called \emph{defended by $S$ in $F$} if for every $(s', s) \in R$, there exists $s'' \in S$ such that $(s'', s') \in R$.  
The set~$\adef_F(S)$ is defined by $\adef_F(S) \eqdef\{ s \mid s \in A, s \text{ is defended by $S$ in $F$}  \}$.

In abstract argumentation, one strives for computing so-called \emph{extensions}, which are subsets~$S \subseteq A$ of the arguments that have certain properties.
The set~$S$ of arguments is called \emph{conflict-free in~$S$} if $(S\times S) \cap R = \emptyset$; $S$ is \emph{admissible in~$F$} if
(1) $S$ is \emph{conflict-free in $F$}, and
(2) every $s \in S$ is \emph{defended by $S$ in $F$}.
Let $S^+_R:=S\cup\{\, a\mid (b,a)\in R, b \in S\, \}$ and $S$ be admissible. 
Then, $S$ is
a) \emph{complete in~$F$} if $\adef_F(S) = S$;
b) \emph{semi-stable in $F$} if no admissible set $S' \subseteq A$ in~$F$ with~$S^+_R\subsetneq (S')^+_R$ exists; and 
c) \emph{stable in~$F$} if every $s \in A \setminus S$ is \emph{attacked} by some $s' \in S$.
A conflict-free set~$S$ is \emph{stage in $F$} if there is no conflict-free set~$S'\subseteq A$ in~$F$ with~$S^+_R\subsetneq (S')^+_R$.
Let $\ALL^*\eqdef\{\conf,\adm, \comp, %
\stab\}$ and $\ALL\eqdef \ALL^* \cup \{\semi, \stag\}$. %
For a semantics~$\sem \in \ALL$, we write $\ext{\sigma}(F)$ for the set of \emph{all extensions} of~$\sem$ in $F$.

Let AF~$F{\,=\,}(A,R)$ be an additionally given framework.
Then, the problem~$\cons_{\sem}$ asks if $\ext{\sigma}(F)\neq\emptyset$, and
$\cred({\sem})$ %
asks for given %
$c{\,\in\,}A$, whether~$c$ is in some $S\in \ext{\sigma}(F)$ (``\emph{credulously} accepted'').

\begin{figure}
  \centering
  \begin{tikzpicture}[scale=.8,arg/.style={circle,fill=black,inner sep=.75mm},y=.6cm]
    \node[arg, label={\bfseries noS}] (C) at (0,0) {};
    \node[arg, label={180:\bfseries T}] (T) at (-1,-1) {};
    \node[arg, label={[label distance=-1mm]0:\bfseries P}] (E) at (0,-1) {};
    \node[arg, label={0:\bfseries W}] (V) at (1,-1) {};
    \node[arg, fill=white, label={270:\color{white}\bfseries D}] (R) at (-.5,-2) {};
    \foreach \f/\t in {C/T,C/E,V/C}{
      \path[-stealth'] (\f) edge (\t);
    }
  \end{tikzpicture}\\[-3em]
  \caption{Argumentation framework for Example~\ref{ex:AF}.}\label{fig:af-simple}
\end{figure}
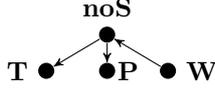

\begin{example}\label{ex:AF} %
  Consider an AF that models aspects of a conference
  submission, Figure~\ref{fig:af-simple}.
  If we have \ns{}ubmission, we cannot \tr{}ravel to a conference and
  \pr{}resent.
  Having a paper \pwr{}ritten up attacks \ns{}ubmission.
  Hence, a stable extension is $E=\{\pwr,\tr,\pr\}$. \qedexample
\end{example}

\paragraph{Answer Set Programs (ASP).} 
We consider a universe~$U$ of propositional \emph{atoms}; a
\emph{literal} is an atom~$a\in U$ or its negation~$\neg a$. A
\emph{program}~$\mathcal{P}$ is a set of \emph{rules} of the form $a_1\por \dots \por a_l \leftarrow b_1,\dots,b_n$, $\pnot c_1, \dots,
\pnot c_m$ where $a_1,\dots,a_l,b_1,\dots,b_n,c_1,\dots, c_m$
are atoms. %
We write
$H(r){=}\{a_1, \dots, a_l\}$, called \emph{head} of $r$,
$B^+(r){=} \{b_1, \dots, b_n\}$, called \emph{positive body} $r$, and
$B^-(r) {=} \{c_1, \dots, c_m\}$, called \emph{negative body} of~$r$.
The atoms in a rule~$r$ or program~$\mathcal{P}$ are given by $\at(r){=}H(r) {\cup} B^+(r) {\cup} B^-(r)$ and
$\at(\mathcal{P}){=}\bigcup_{r\in \mathcal{P}} \at(r)$, respectively.
We sometimes view a rule~$r\in\mathcal{P}$ as a set~$H_r\cup B_r^-\cup \bigcup_{x\in B_r^+}\{\neg x\}$ of literals.
Consider digraph~$D_\mathcal{P}$ of a program~$\mathcal{P}$ which has vertices~$\at(\mathcal{P})$ and a directed edge $(x,y)$ between any two atoms
$x,y \in \at(\mathcal{P})$ if there is $r\in \mathcal{P}$ with $x \in H(r)$
and $y\in B^+(r)$.
Then, $\mathcal{P}$ is \emph{tight} if %
there is no cycle in $D_\mathcal{P}$~\cite{KanchanasutStuckey92,Ben-EliyahuDechter94}.
We refer by \class{Disj} to the class of \emph{(disjunctive) programs}.
By \Normal we mean programs containing only rules $r$ with $\Card{H(r)}\leq 1$, so called \emph{normal} rules.
The class~\Tight refers to tight programs.

A set~$M$ of atoms \emph{satisfies} a rule~$r$ if
$(H(r)\,\cup\, B^-(r)) \,\cap\, M \neq \emptyset$ or
$B^+(r) \setminus M \neq \emptyset$.  $M$ is a \emph{model} of~$\mathcal{P}$ if
it satisfies all rules of~$\mathcal{P}$, we write $M \models \mathcal{P}$ for short.
The \emph{(GL) reduct} of a program~$\mathcal{P}$ under a
set~$M$ of atoms is program~$\mathcal{P}^M \coloneqq \SB H(r) \leftarrow B^+(r) \SM r \in \mathcal{P}, M \cap
B^-(r) = \emptyset\SE$~\cite{GelfondLifschitz91}.
$M$ is an \emph{answer set} %
  program
  $\mathcal{P}$ if $M$ is a $\subseteq$-minimal model of~$\mathcal{P}^M$. %
For tight programs this is equivalent to~$M \models \mathcal{P}$ such that
every~$a\in M $ is \emph{justified by~$\mathcal{P}$},~i.e., there has to
exist~$r\in\at(\mathcal{P})$ with~$a\in H(r)$, $M \supseteq B^+(r)$
and~$M \cap B^-(r)=\emptyset$~\cite{LinZhao03}.
Deciding whether a program has an answer set
(\emph{consistency problem}) is
$\SigmaP2$-complete~\cite{EiterGottlob95}. %

\paragraph{QBFs.}
Let $\ell$ be a positive integer, called \emph{(quantifier) rank}, and $\top$ and $\bot$ be the constant always evaluating to $1$ and $0$, respectively.
A conjunction or disjunction of literals is called a term or clause, respectively.
For a Boolean formula~$F$, we write $\var(F)$ for the variables occurring in~$F$ and $F(X_1,\ldots,X_\ell)$ to indicate that $X_1,\ldots,X_\ell\subseteq \var(F)$.
A \emph{quantified Boolean formula (QBF)}~$\phi$ is %
of the form $\phi=Q_1 X_1.Q_2 X_2.\cdots Q_\ell X_\ell. F(X_1,\dots,X_\ell)$, where for $1\leq i\leq \ell$, we have $Q_i\in\{\forall,\exists\}$ and $Q_i \neq Q_{i+1}$, the $X_i$ are disjoint, non-empty sets of Boolean variables, and $F$ is a Boolean formula. We let $\matr(\phi)\dfn F$ and %
we say that $\phi$ is \emph{closed} if $\var(F)= \bigcup_{i \in \ell}X_i$.
We evaluate $\phi$ by $\exists x.\phi\equiv \phi[{x\mapsto 1}]\lor\phi[{x\mapsto 0}]$ and $\forall x.\phi\equiv \phi[{x\mapsto 1}]\land\phi[{x\mapsto 0}]$ for a variable~$x$.
We assume that $\matr(\phi)=\psi_{\text{CNF}} \wedge \psi_{\text{DNF}}$, 
where $\psi_{\text{CNF}}$ is in CNF (conjunction of clauses) and $\psi_{\text{DNF}}$ is in DNF (disjunction of terms).

Syntactically, we often view a Boolean CNF formula~$\psi_{\text{CNF}}$ and a DNF formula~$\psi_{\text{DNF}}$ as a set of sets of literals. %
Then, depending on $Q_\ell$, either $\psi_{\text{CNF}}$ or $\psi_{\text{DNF}}$ is optional, more precisely,  $\psi_{\text{CNF}}$ might be $\top$, if $Q_\ell=\forall$, and $\psi_{\text{DNF}}$ is allowed to be $\top$, otherwise. 
The problem $\ell\hy\QSAT$ asks, given a closed QBF $\phi=\exists X_1.\phi'$ of
rank~$\ell$, whether $\phi\equiv1$ holds.

\paragraph{Tree Decompositions and Treewidth.}
For a rooted (directed) tree~$T=(N,A)$ with \emph{root~$\rootOf(T)$} and a node~$t \in N$, let $\children(t)$ be the set of all nodes~$t^*$, which have an edge~$(t,t^*) \in A$. 

Let $G=(V,E)$ be a graph. 
A \emph{tree decomposition (TD)} of a graph~$G$ is a pair $\TTT=(T,\chi)$, where $T$ is a rooted tree, and $\chi$ is a mapping that assigns to each node $t$ of $T$ a set $\chi(t)\subseteq V$, called a \emph{bag}, such that:
(1.) $V=\bigcup_{t\text{ of }T}\chi(t)$ and $E \subseteq\bigcup_{t\text{ of }T}\{ \{u,v\} \mid u,v\in \chi(t)\}$, 
(2.) for each $s$ lying on any $r$-$t$-path: $\chi(r) \cap \chi(t) \subseteq \chi(s)$.
Then, define $\width(\TTT) \eqdef \max_{t\text{ of }T}\Card{\chi(t)}-1$.  
The \emph{treewidth} $\tw{G}$ of $G$ is the minimum $\width({\TTT})$ over all TDs $\TTT$ of $G$. 
Observe that for every vertex~$v\in V$, there is a unique node~$t^*$ with~$v\in \chi(t^*)$ such that either~$t^*=\rootOf(T)$ or there is a node $t$ of~$T$ with~$\children(t){=}\{t^*\}$ and $v\notin\chi(t)$.
We refer to the node~$t^*$ by~$\forg(v)$.
We assume TDs~$(T,\chi)$, where for every node $t$,  we have $\Card{\children(t)}\leq 2$, obtainable in linear time without width increase~\cite{BodlaenderKoster08}. 

For a given QBF~$\phi$ with~$\matr(\phi)=\psi_{\text{CNF}} \wedge \psi_{\text{DNF}}$, we define the \emph{primal graph~$\primal{\phi}=\primal{\matr(\phi)}$}, whose vertices are~$\var(\matr(\phi))$. Two vertices of~$\primal{\phi}$ are adjoined by an edge,
whenever corresponding variables share a clause or term of~$\psi_{\text{CNF}}$ or~$\psi_{\text{DNF}}$.
Let $\tower(i,p)$ be $\tower(i-1,2^p)$ if $i>0$ and $p$ otherwise.
Further, we assume that~$\poly(n)$ is any polynomial for given positive integer~$n$.

\begin{proposition}[\citey{Chen04a}]\label{prop:qbf}
For any arbitrary QBF~$\phi$ of quantifier rank~$\ell>0$, the problem $\ell\hy\QSAT$ can be solved in time $\tower(\ell, \mathcal{O}(\tw{\primal{\varphi}}))\cdot\poly(\Card{\var(\phi)})$. %
\end{proposition}

Under \emph{exponential time hypothesis~(ETH)}~\cite{ImpagliazzoPaturiZane01},
this cannot be significantly improved. %

\begin{proposition}[\citey{FichteHecherPfandler20}]\label{qbf:lb}
Under ETH, for any QBF~$\varphi$ of quantifier rank~$\ell>0$, $\ell\hy\QSAT$ cannot be solved in time~$\tower(\ell, {o}(\tw{\primal{\varphi}}))\cdot\poly(\Card{\var(\varphi)})$.
\end{proposition}

\paragraph{TDs for AFs.}
Consider for AF~$F=(A,R)$ the \emph{primal graph}~$\primal{F}$, where we
simply drop the direction of every edge,~i.e., $\primal{F}=(A,R')$ where
$R' \eqdef \{ \{u,v\} \,|\, (u,v) \in R\}$.
For any TD~$\mathcal{T}=(T,\chi)$ of~$\primal{F}$ and any node~$t$ of~$T$, we let~$A_t\eqdef A\cap \chi(t)$ be the \emph{bag arguments of~$t$} and $R_t\eqdef R \cap \{(a,b) \mid a,b\in\chi(t)\}$ be the \emph{bag attacks of~$t$}.

\newcommand{\RC}[0]{\text{RC}\xspace}

\section{Rejection Augmented AFs (\RAF{}s)}\label{sec:CCFs}
Next, we provide a definition for the syntax and semantics of
rejection conditions in abstract argumenta\-tion frameworks.

\begin{definition}\label{def:syntax:CC}
  A \emph{rejection augmented AF (\RAF)} is a
  triple~$\CF = (A,R,C)$ where 
  (1.) $(A,R)$ is an AF, and 
  (2.) $C\colon A \rightarrow \Pi$ is a labeling, called 
  \emph{rejection conditions (\RC{}s)}, that maps every argument in
  the framework to a constraint $\mathcal{P}\in\Pi$ 
  where $\Pi$ is a set of ASP programs.
\end{definition}

For brevity, we lift, for a set~$E\subseteq A$ of arguments, the
\RC{}s of~$E$ to~$C(E)\eqdef \bigcup_{e\in E}C(e)$.

We say that~$C$ is $\props$, if $\Pi$ is a set of CNF formulas instead
of programs.
We say $C$ is $\simple$ if $\var(C(A))=A$ and
$C$ is propositional.
If~$C(A)$ is a tight or normal program,
$C$ is called $\tight$ or $\normal$, respectively; otherwise
$C$ is $\disjs$.

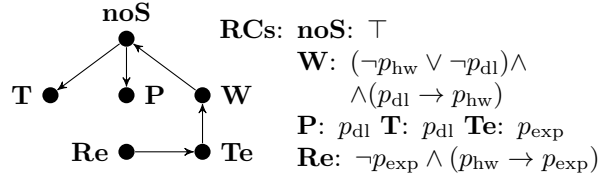
\begin{figure}
  \centering
  \begin{tikzpicture}[arg/.style={circle,fill=black,inner sep=.75mm},y=.75cm]
    \path[use as bounding box] (-1.6,-2.4) rectangle (6.2,.9);
    \node[arg, label={\bfseries noS}] (C) at (0,0) {};
    \node[arg, label={180:\bfseries T}] (T) at (-1,-1) {};
    \node[arg, label={0:\bfseries P}] (E) at (0,-1) {};
    \node[arg, label={0:\bfseries W}] (V) at (1,-1) {};
    \node[arg, label={0:\bfseries Te}] (S) at (1,-2) {};
    \node[arg, label={180:\bfseries Re}] (Re) at (0,-2) {};

    \foreach \f/\t in {C/T,C/E,V/C,Re/S,S/V}{
      \path[-stealth'] (\f) edge (\t);
    }
    
    \node[anchor=north west,align=left,text width=5cm] at (1.1,0.5) {\textbf{\RC{}s}: 
      $\brm{noS}$: $\top$\\
      \phantom{\textbf{\RC{}s}: }$\brm{W}$: $(\lnot p_{\text{hw}}\lor \lnot p_{\text{dl}})\land$\\
      \phantom{\textbf{\RC{}s}: }\phantom{$\brm{W}$: }$\land(p_{\text{dl}}\to p_{\text{hw}})$\\
      \phantom{\textbf{\RC{}s}: }$\brm{P}$: $p_{\text{dl}}$ $\brm{T}$: $p_{\text{dl}}$ $\brm{Te}$: $p_{\text{exp}}$\\
      \phantom{\textbf{\RC{}s}: }$\brm{Re}$: $\lnot p_{\text{exp}}\land(p_{\text{hw}}\to p_{\text{exp}})$
    };
  \end{tikzpicture}
  \caption{\RAF modeling an excerpt of every day research.}\label{fig:CCF-complex}
\end{figure}

\paragraph{Semantics.} 
For an argumentation semantics, we naturally extend the notion to an extension for \RAF{}s $\CF=(A,R,C)$.
Intuitively, the mapping~$C$ gives rise to a program/formula that specifies the \RC{}s for each argument.
Ultimately, \RC{}s of an extension must be invalidated. 

\begin{definition} \label{def:semantics}
Let $\CF{=}(A,R,C)$ be a \RAF and~$\sem\in\ALL$. %
Then, $E\subseteq A$ is a \emph{$\sem$-extension (of~$\CF$)} if $E\in\ext{\sigma}(A,R)$ %
and~$C(E) \cup E \cup \bigcup_{a\in A\setminus E}\{\bot\leftarrow a\}$ %
 is inconsistent.
\end{definition}
A counterclaim is a formula that prohibits the choice of an argument unless it is false, meaning that the counterclaim is not true. 
Note that arguments with $\top$ counterclaims could also be accepted, as long as the extension's combined counterclaims vanish (are invalidated). 

\begin{example}\label{ex:CCF-complex} 
	Consider the AF from Example~\ref{ex:AF}, which we slightly extend.
  \pwr{}riting a paper is
  attacked by \tea{}aching, which is attacked by doing
  \re{}search.
  Now, we include \dl{}eadlines into the scenario, but accept arguments
  $\{\pwr,\tr,\pr\}$ regardless of deadlines.
  Intuitively, if one is not hard-working or if there is no deadline,
  one does not write a paper.
  So each of these two ``counters'' the argument itself.  Yet, the
  deadline will make us hard-working.
  Being hard-working enables experiments which are necessary in
  empirical research.
  We add a \emph{rejection condition} (\RC{}s) to each argument (see Fig.~\ref{fig:CCF-complex}).
  The \RC{}s for \tr{} and \pr{} are a single proposition
  $p_{dl}$ modeling the RC that there is deadline.
  For \pwr, we take $\lnot p_{dl}$ modeling that one would not write a paper if there are no \dl{}eadlines.
  Now, choosing $\{\re, \pwr,\tr,\pr\}$ yields a stable extension as
  $p_{dl}\land\lnot p_{dl}$ is false.
  Consequently, $C(E)=\allowbreak(\lnot p_{\text{hw}}\lor \lnot
  p_{\text{dl}})\land(p_{\text{dl}}\to p_{\text{hw}})\land
  p_{\text{dl}}\land \lnot p_{\text{exp}}\land(p_{\text{hw}}\to
  p_{\text{exp}})\equiv \bot$. 
  Notice that adding arguments for (no) deadlines will result in loosing the extension $\{\pwr, \tr, \pr\}$.
  \qedexample
\end{example}

Interestingly, the empty extension
($\emptyset$) is no longer a valid extension for \RAF{}s.  
This follows from the fact that (empty) \RC{}s yield
$C(E)=\emptyset$ which produces a consistent $E$.
This is in line with our intuition, as
$C(E)$ admits the empty counterclaim.
Note that the arguments may be used as parts of the programs
in~$C$.  
If an argument $a$ has unconditional \RC{}s then $C(a) = \emptyset
\equiv \top$.  Again, this is consistent with our intuition that
$a$ does not affect $C(E)$.

\paragraph{Using ASP Programs.}
The use of ASP for modeling \RC{}s enables non-monotonicity and
advances its natural modeling. So, it may well be that a candidate~$E$
admits \RC{}s, but a superset of~$E$ does~not, which we demonstrate
below.

\begin{example}\label{ex:asp}
  Figure~\ref{fig:asp} provides an \RAF~$\CF=(A,R,C)$ where $\bot$ is modeled via $p\land\lnot p$. 
  Let $R$ be the attack relation. Then, 
  $\ext{\adm}(A,R)=\{\emptyset$, $\{a\}$, $\{b\}$, $\{d\}$, $\{a,b\}$, $\{a,d\}\}$.
  Only~$\{a,b\}$ and~$\{d\}$ are admissible extensions of~$\CF$, since the others admit non-empty \RC{}s. 
  Due to non-monotonicity of ASP, $\{a,d\}$ is not an admissible extension of~$\CF$. %
\end{example}
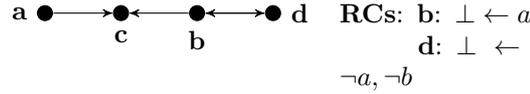
\begin{figure}
  \centering
  \begin{tikzpicture}[arg/.style={circle,fill=black,inner sep=.75mm}]
    \node[arg, label={270:\bfseries c}] (c) at (0,0) {};
    \node[arg, label={180:\bfseries a}] (a) at (-1,0) {};
    \node[arg, label={270:\bfseries b}] (b) at (1,0) {};
    \node[arg, label={0:\bfseries d}] (d) at (2,0) {};

    \foreach \f/\t in {a/c,b/c, b/d,d/b}{
      \path[-stealth'] (\f) edge (\t);
    }
    
    \node[anchor=north west,align=left,text width=3.3cm] at (2.75,0.25) {\textbf{\RC{}s}:
      $\brm{b}$: $\bot \leftarrow a$\\
      \phantom{\textbf{\RC{}s:} }$\brm{d}$: $\bot \leftarrow \neg a, \neg b$\\
    };
  \end{tikzpicture}\\[-1em]
  \caption{ASP enables non-monotonicity for \RC{}s.}\label{fig:asp}
\end{figure}

Now, if $C({a})\cup \{\bot \leftarrow a\}\}$ is inconsistent for each
$a\in A$, the \RAF-semantics collapses to the AF-semantics and
extensions coincide (modulo the empty extension). This 
motivates more detailed considerations for the following section.

\subsection{Simulating Frameworks}
\label{sec:expressiveness}

To understand \RAF{}s and their expressivity, we %
simulate AFs~\cite{Dung95a} and CAFs~\cite{Coste-MarquisDevredMarquis06a}.
Therefore, we require a notation with only reasonable
overhead.

\begin{definition}
  Let $\mathcal F\in \{\text{AF, CAF}\}$ be an argumentation
  formalism.
  Then, we say that \RAF{}s \emph{simulate} $\mathcal F$, in symbols
  \RAF{}s $\succeq \mathcal{F}$, if for every
  semantics~$\sigma\in \ALL$ and for every framework
  $F\in \mathcal F $, there is an \RAF~$F'$ such that
  $\ext{\sigma}(F)\setminus\{\emptyset\} =\ext{\sigma}(F')$ and~$F'$
  can be constructed in polynomial time.
  We write $\succ$ if the relation is strict.
\end{definition}

Clearly, \RAF{}s $\succeq$ AFs, e.g., by setting the \RC{}s to false.	
\begin{observation}\label{af:CCF}
	\RAF{}s $\succeq$  AFs. 
\end{observation}
Later, we will see in Thm.~\ref{thm:hard} that unless
$\NP = \SigmaP{2}$, \RAF{}s $\succ$ AFs.
Moreover, for an \RAF $G= (A,R,C)$ and AF $F=(A,R)$, it is easy to
observe that
$\ext{\sigma}(G) \subseteq \ext{\sigma}(F)\setminus\{\emptyset\}$ for
each $\sigma\in \ALL$.  The converse does not hold (see
Example~\ref{ex:asp}).  By definition, the following relation between
semantics %
manifests for \RAF{}s. %

\begin{observation}\label{sem:relation}
	Given an \RAF $G$, then
	$\ext{\stab(G)} \subseteq \ext{\semi(G)} \subseteq$ $\ext{\pref(G)}\subseteq \ext{\comp(G)} \subseteq \ext{\adm(G)}$
	and
	$\ext{\stab(G)} \subseteq \ext{\stag(G)} \subseteq \ext{\conf(G)}.$
	
\end{observation}
\paragraph{Comparison to CAFs.} 
We briefly recall CAFs~\cite{Coste-MarquisDevredMarquis06a}.  A CAF is
a triple~$\CAF \dfn (A,R,\phi)$ where $(A,R)$ is an AF and $\phi$ is a
propositional formula over $A$.  For $E \subseteq A$, the
\emph{completion} of $E$ is
$\hat E = E \cup \{\neg a \mid a \in A \setminus E\}$.  Let
$E\subseteq A$, then $E$ is $\CAF$-admissible if $E$ is admissible for
$(A,R)$ and $\hat E\models \phi$, $\CAF$-preferred if $E$ is maximal
for set-inclusion among the $\CAF$-admissible sets, and 
$\CAF$-stable if $E$ is conflict-free, $\hat E\models \phi$ and
attacks all arguments in $A\setminus E$.  The remaining semantics for
CAFs are similarly defined.
Intuitively, \RAF{}s employ \RC{}s for each
argument %
and consider a larger set of variables, instead CAFs employ one
constraint for the whole AF. %
We establish that \RAF{}s strictly simulate CAFs.
\begin{lemma}\label{caf:CCF}
	\RAF{}s $\succeq$ CAFs.
\end{lemma}
\begin{proof}
  Let $(A,R,\phi)$ be a CAF.  
  For $\sigma \in \{\adm,\stab,\comp\}$, we set $C(a)= \neg \phi$ for all $a\in A$ s.t.\ 
  the corresponding $\sigma$-extensions coincide. 
  For $\sigma \in \{\pref,\semi,\stag\}$, %
  employ admissible/conflict-free AFs and encode %
  additional properties into \RC{}s. %
  
  For preferred semantics, we let $C(a)= \neg \phi \lor \psi_\pref$ for each $a\in A$ where $\psi_\pref$ is a formula over $A$ and further uses fresh variables $\{a',a'', \mid a\in A\}$. %
  Intuitively, $\psi_\pref$ encodes the existence of a CAF-admissible extension $D$ as a counter example to a CAF-preferred set $E$. 
  That is, (1) ``$D$ is admissible'', (2)  ``$\hat D \models \phi$'' and (3) ``$D\supsetneq E$''. %
  \begin{enumerate}
  	\item $\bigwedge_{(a,b)\in R}(\neg a'\lor \neg b')\land \bigwedge_{(b,a)\in R}( \bigvee_{(c,b\in R)}(c'\lor \neg a'))$ 
  	\item $(\bigwedge_{a\in A}(a' \leftrightarrow a'') \land \phi_{[a\mapsto a'']}$
  	\item $\bigwedge_{a\in A}(a\rightarrow a') \land (\bigvee_{a\in A} (a'\land \neg a))$
  \end{enumerate}
  Consequently, the preferred extensions in CAF are exactly the admissible sets in CCF such that $C(E)\cup E \cup \{\neg a \mid a\in A\setminus E\}$ is inconsistent.
  $E$ is preferred in CAF iff $E$ is admissible, $\hat E\models \phi$ and maximal among admissible CAF-extensions %
  iff E is CCF-admissible. %
  The inconsistency of $E \land (\bigwedge_{a\in A\setminus E} \neg a) \land \neg \phi$ ensures that $E$ is CAF-admissible, whereas, the inconsistency of $E \land (\bigwedge_{a\in A\setminus E} \neg a) \land \psi_\pref$ ensures that $E$ has no CAF-admissible strict superset $D$.
  
  For $\semi$, we similarly let $C(a)= \neg \phi \lor \psi_{\semi}$ for each $a\in A$ where $\psi_\semi$ is a formula over $A$ and further uses fresh variables $\{a',a'', d_a,d_a' \mid a\in A\}$.
  Intuitively, $\psi_\semi$ encodes the existence of a CAF-admissible set $D$ as a counter example to a $\semi$-extension $E$.
  That is, (1) D is admissible, (2) $\hat D\models \phi$ and (3) $D^+_R \supsetneq E^+_R$.
  \begin{enumerate}
  	\item[1.+2.] same as before.
  	\item[3a.] $(\bigwedge_{a\in A} d_a \leftrightarrow \bigvee_{(b,a)\in R} b) \land (\bigwedge_{a\in A'} d'_a \leftrightarrow \bigvee_{(b,a)\in R} b')$
  	\item[3b.] $\bigwedge_{a\in A}(a\rightarrow (a' \lor d'_a) \land \bigwedge_{a\in A}(d_a\rightarrow (a' \lor d'_a)$
  	\item[3c.] $\bigvee_{a\in A} ((a'\lor d_a')\land (\neg a\land \neg d_a'))$
  \end{enumerate}
	Formula~(3a.) defines which arguments are defended ($d_a, d_a'$) and (3b.+3c.) state that $E^+_R\subsetneq D^+_R$. %
	As before, semi-stable extensions in CAF are the admissible sets in CCF for which the \RC{}s are invalidated.
	
	Finally, for $\stag$, we reuse $\psi_\semi$ and weaken the first condition (Formula 1.) so that $D$ is only conflict-free.
\end{proof}
Again, by a complexity argument later, in
Theorem~\ref{thm:hard}, we obtain that \RAF{}s $\succ$ CAFs assuming
that $\NP \neq \SigmaP{2}$.

\paragraph{\RAF{}s for Extensions In-Between Semantics.}
We argue that RCs are expressive enough to model certain AF-semantics as additional constraints.
We fix some notation first. 
For semantics $\sigma_1,\sigma_2$, we say that $\sigma_1$ is \emph{stricter} than $\sigma_2$ (and $\sigma_2$ is \emph{laxer} than $\sigma_1$) if $\ext{\sigma_1}(F)\subseteq \ext{\sigma_2}(F)$ for every AF $F$. For example, $\adm$ is laxer then $\pref$ or $\stab$, and stricter than $\conf$.
The idea is as follows: consider an AF $F$ where an extension for a laxer semantics exists whereas, an extension for a stricter semantics does not exist on $F$.
However, we are still interested in getting ``close'' to the stricter semantics and approximate an extension on a subset of $F$, as discussed in Example~\ref{ex:intro-hybrid}.

\begin{example}\label{ex:intro-hybrid}
  Consider the AF $F$ as depicted in Figure~\ref{fig:intro-hybrid}, which
  has no stable extension.
  Still, $\{a\}$  and $\{b\}$ are stable when restricting the AF to $\{a,b,c\}$. \qedexample
\end{example}

In the following, we formalize this intuition and prove that we can use \RC{}s to model such requirements.

\begin{definition}[Twofold Extension]
	Let $\sigma_1$ and $\sigma_2$ be two semantics.	
	Furthermore, let $F=(A,R)$ be an AF and call $S\subseteq A$ a \emph{shrinking} of~$A$. 
	We say that $E$ is a \emph{twofold} $(\sigma_1,\sigma_2)$-extension of $(F,S)$ if $E$ is a $\sigma_1$-extension in $F$ and $E\cap S $ is a $\sigma_2$-extension in the induced sub-framework $F[S]$,~i.e., $F[S]=(S,R')$ where $ R'=\{\;(u,v)\mid (u,v)\in R, u\in S,v\in S\;\}$.
\end{definition}

\begin{example}
	Consider AF~$F$ from Example~\ref{ex:intro-hybrid}.
	Then, $\{c\}$ is a twofold ($\conf,\pref$)-extension for the shrinking $\{c,d,e\}$ of $A$. 
	Moreover, $\{a\}$ and $\{b\}$ are twofold ($\adm,\stab$)-extensions for the shrinking $\{a,b,c\}$ of $A$. \qedexample
\end{example}

Next, we illustrate that the additional criterion imposed by the \emph{second} semantics on a shrinking can be simulated (for stable extensions) via simple \RC{}s. %

\begin{observation}\label{thm:hybrid}
  Let $F=(A,R)$ be an AF, $S\subseteq A$ be a shrinking of $A$ and
  $\sigma\in \ALL$ be a semantics.  Then, there is an \RAF $\CF$ such
  that the $\sigma$-extensions in $\CF$ are exactly the twofold $(\sigma,\stab)$-extensions of $(F,S)$.
\end{observation}
\begin{proof} %
	Let $A'\dfn \{\;a'\mid a\in A\;\}$ denote a (fresh) set of propositions.
	Then, we let $\CF\dfn(A,R,C)$ be the \RAF with \RC{}s~$C(a)\dfn a'\land \bigwedge_{(a,b)\in R} b' \land \bigvee_{x\in S}\neg x'$ for $a\in S$, and $C(a)\dfn \top $ otherwise. 
	Then, $E$ is a twofold ($\sigma,\stab$)-extension of $(F,S)$ if and only if $E$ is a $\sigma$-extension in $\CF$ and $C(E)\cup E \cup \bigcup_{a\in A\setminus E}\{\neg a\}$ is inconsistent.
	
	``$\Rightarrow$'': If $E\cap S$ is a $\stab$-extension in $F[S]$, then for each $x\in S$, either $x\in E$ or there is $y\in E\cap S$ for some $y$ such that $(y,x)\in R$. 
	This implies that $\bigwedge_{e\in E\cap S} (e' \land \bigwedge_{(e,b)\in R} b')\land \bigvee_{x\in S}\neg x' $ is inconsistent. 
	Consequently $C(E)\cup E \cup \bigcup_{a\in A\setminus E}\{\neg a\} \equiv \bot$.
	
	 ``$\Leftarrow$'': If E is a $\sigma$-extension in $\CF$, then $C(E)\cup E \cup \bigcup_{a\in A\setminus E}\{\neg a\} \equiv \bot$. 
	 But this is equivalent to the expression $\bigwedge_{e\in E\cap S} (e'\land \bigwedge_{(e,b)\in R} b') \land \bigvee_{x\in S}\neg x' \equiv \bot$ (since $A\setminus E$ contains negative literals and the only positive literals are formed over sets disjoint from $A\setminus E$). 
	As a result, $E\cap S$ is stable in $F[S]$.
\end{proof}

\begin{figure}
  \centering
  \begin{tikzpicture}[arg/.style={circle,fill=black,inner sep=.75mm,y=5em}]
    \node[arg, label={\bfseries c}] (c) at (-1,-.5) {};
    \node[arg, label={180:\bfseries a}] (a) at (-2,-1) {};
    \node[arg, label={0:\bfseries b}] (b) at (0,-1) {};

    \node[arg, label={0:\bfseries e}] (e) at (2,-1) {};
    \node[arg, label={0:\bfseries d}] (d) at (2,-0.5) {};

    \foreach \f/\t in {a/b,b/a,a/c,b/c, e/d}{
      \path[-stealth'] (\f) edge (\t);
    }
    \path[-stealth'] (e) edge[loop left,>=stealth'] (e);

  \end{tikzpicture}
  \caption{``Sub-framework'' of an AF with stable extension.}\label{fig:intro-hybrid}
\end{figure}
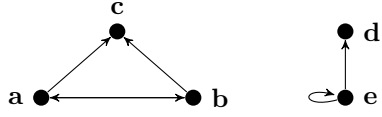

\section{How Hard is Rejection?}\label{sec:complexity}
In this section, we study the hardness of rejection.
Interestingly, already the use of auxiliary variables
makes the problem significantly harder.

\begin{theorem}\label{thm:hard}
The problem $\cons_{\conf}$ is $\Sigma_i^p$-complete,
for simple ($i=1$), propositional/tight ($i=2$), and disjunctive ($i=3$) \RC{}s.
\end{theorem}%

\begin{proof}[Proof (Idea)]
Membership: It suffices to nondeterministically guess a conflict-free set $E$  and then evaluate $C(E)\cup E$ directly.
The evaluation of $C(E)\cup E$ is done in polynomial time (case simple) that can be processed directly yielding $\NP$, in $\NP$-time (case propositional/tight) yielding $\SigmaP2$, and in $\SigmaP2$-time (case disjunctive) yielding membership in $\SigmaP3$.

Next we show hardness.
``$\props$'': %
Reduce from~$\TQSAT$, where for any~$Q=\exists X. \forall Y. \varphi(X,\allowbreak Y)$ with $\varphi$ being in 3-DNF; assume only terms over both~$X$ and~$Y$.
Then, we construct the following \RAF~$\CF\eqdef(X\cup X', R, C)$ (see Example~\ref{ex:prop-3dnf-ext}):
$R\eqdef\{(x,x'), (x', x)) \mid x\in X\}$, 
$C(v) \eqdef \{\neg c %
\mid c \in \varphi, %
v\in \var(c)\}$  for~$v=x$ with~$x\in X$, and~$v=x'$ with~$x'\in X'$. %
Clearly, $Q$ is valid iff~$\cons_{\conf}(\CF){\neq}\emptyset$. %

``$\simple$'': %
Reduce from~$\BSAT$, where we let~$\varphi$ be any CNF with~$X=\var(\varphi)$.
Assume for a variable~$x$ and a set~$X$ of variables that~$x'$ is the %
``copy'' %
of~$x$.
Construct the %
\RAF~$\CF\eqdef(X\cup X' \cup \{v_c\mid c\in \varphi\}, R, C)$:
$R\eqdef\{(x,x'), (x', x)) \mid x\in X\}$, 
$C(x) \eqdef \{v_c\rightarrow\neg c %
\mid c \in \varphi, %
x\in \var(c)\}$ for~$v=x$ with $x\in X$, and~$v=x'$ with~$x'\in X'$. %
Show that~$\varphi$ is satisfiable iff~$\CF$ has a conflict-free extension.

``$\tight$'':  %
Reduce from~$\TQSAT$, where for any~$Q=\exists X. \forall Y. \varphi(X,Y)$ with $\varphi$ in DNF, assume only terms over both~$X$ and~$Y$.
Then, construct the %
\RAF~$\CF\eqdef(X\cup X', R, C)$:
$R\eqdef\{(x,x'), (x', x)) \mid x\in X\}$, 
$C(a) \eqdef \{y \leftarrow \neg y', y'\leftarrow\neg y, \leftarrow B \mid c \in \varphi, B=\bigcup_{\neg x\in c} x' \cup \bigcup_{y\in c\cap\var(c)} y, a\in \var(B), y\in \var(c)\cap Y\}$ for~$a\in X\cup X'$.
Thus~$Q$ is valid iff~$\CF$ has a conflict-free extension. %

``$\disjs$'': %
We reduce from~$3\hy\QSAT$, where for any~$Q=\exists X. \forall Y. \exists Z. \varphi(X,\allowbreak Y,Z)$ with $\varphi$ in CNF, we assume only terms over~$X$ and~$Y$.
We let %
\RAF $\CF\eqdef(X\cup X', R, C)$:
$R\eqdef\{(x,x'), (x', x)) \mid x\in X\}$, 
$C(a) \eqdef \{y\leftarrow \neg y', y'\leftarrow \neg y, z\vee z'  \leftarrow, z \leftarrow s, z' \leftarrow s, \leftarrow \neg s, s \leftarrow B \mid  c \in \varphi, B=\bigcup_{\neg x\in c} x \cup \bigcup_{y\in c\cap\var(c)} y', a\in \var(B), z\in\var(c)\cap Z, y\in\var(c)\cap Y\}$ for~$a\in X\cup X'$.
$Q$ is valid iff~$\cons_{\conf}(\CF){\neq}\emptyset$. %
\end{proof}

Consequently, hardness follows for semantics with conditions stronger than conflict-freeness.

\begin{corollary}\label{cor:cred}
Problems $\cons_\sigma$ and $\cred_\sigma$ for $\sigma\in\ALL^*$ are $\Sigma_i^p$-complete,
for simple ($i=1$), propositio\-nal/tight ($i=2$), and disjunctive ($i=3$) \RC{}s.
\end{corollary}
\begin{proof}[Proof (Idea)]
Membership holds similarly to the theorem, as for the semantics it suffices to guess an extension candidate and verify it in polynomial-, $\NP$-, or $\SigmaP2$-time, respectively.
Hardness vacuously carries over from the theorem.
\end{proof}

Next, we show a proof example and establish membership.

\begin{example}\label{ex:prop-3dnf-ext}
	Consider the valid 3-DNF $\exists x_1\exists x_2\forall y(x_1\land\lnot x_2\land y)\lor(x_1\land\lnot x_2\land\lnot y)$. 
	Let $d\eqdef(\lnot x_1\lor x_2\lor\lnot y)\land(\lnot x_1\lor x_2\lor y)$. 
	Following Theorem~\ref{thm:hard} ($\props$), we get $C(x_1){=}C(x_2){=}C(x_1'){=}C(x_2'){=}d$. 
	Hence, $E\eqdef\{x_1,x_2'\}$ is a conflict-free extension of~$F$ since $C(E)\cup \{x_1,x_2', \neg x_1', \neg x_2\}%
	\equiv y\land\lnot y\equiv \bot$ is unsatisfiable. \qedexample
\end{example}

\begin{theorem}\label{thm:membership}
The problems~$\cons_\stag$ and~$\cons_\semi$ are in $\Sigma_i^p$, for simple ($i=1$), propositional/tight ($i=2$), or disjunctive ($i=3$) \RC{}s.
\end{theorem}
\begin{proof}
  Given an \RAF~$\CF=(A,R,C)$, compute whether $\ext{\sigma}(\CF)\neq\emptyset$ for stage and semi-stable semantics. %
  For existence of such an extension (and not to directly compute it) it suffices to nondetermini\-sti\-cally guess an admissible set $E$ and then evaluate $C(E)\cup E$ directly. 
  If $E$ was not maximal as required (semi-stable/stage) and $E$ is an extension, then a larger set would falsify $C(E)\cup E$ only more as more arguments and accordingly \RC{}s would be added.
  The evaluation of $C(E)\cup E$ is done in polynomial time (case simple) that can be processed directly yielding $\NP$, in $\NP$-time (case propositional/tight) yielding $\SigmaP2$, and in $\SigmaP2$-time (case disjunctive) yielding $\SigmaP3$.
  \end{proof}

A higher complexity for credulous reasoning is shown for more involved semantics in the next result.

\begin{theorem}\label{thm:credsemistab}
The problems
$\cred_{\semi}$,~$\cred_{\stag}$ are $\Sigma_{i+1}^p$-hard, for simple ($i{=}1$), propositional/tight ($i{=}2$), or disj.\ ($i{=}3$)~\RC{}s.
\end{theorem}
\begin{proof}
	``Lower Bounds'': Analogously as before, given a set of arguments $X$, denote by $\hat X\eqdef\{\hat x\mid x\in X\}$. 
	We adapt an reduction by Dvo{{\v r}{\'a}}k and Woltran~\cite{DvorakWoltran10} that shows credulous reasoning for AFs under semi-stable/stage semantics is $\SigmaP2$-complete.
	We reduce from a $\QSAT$-instance $\Phi=\forall Y.\exists Z.\bigwedge_{c\in \mathcal C}c$ to an \RAF $\CF\eqdef(A,R,C)$, where
        \begin{align*}
          A \eqdef\, &\{t,t', b\}\cup \mathcal C
                       \cup Y\cup Y'\cup \hat Y\cup\hat Y'\cup Z\cup Z',\\
          R \eqdef\, &\{(c,t)\mid c\in \mathcal C\}\cup  \{(\ell,c)\mid c\in \mathcal C, \ell\in c\}\cup \\
                     & \{(x,x'),(x',x)\mid x\in Y\cup Z\}\,\cup\\ 
                     &\{(y,\hat y),(y',\hat y'),(\hat y,\hat y),(\hat y',\hat y')\mid y\in Y\}\cup\\
                     & \{(t,t'),(t',t),(t,b),(b,b)\}.
	\end{align*} %
	Then $t'$ is contained in no semi-stable/stage extension iff~$\Phi$ is valid, thereby solving the complement of~$\cred_{\semi}$/$\cred_{\stag}$. 
	For simple \RC{}s, $C(a)\eqdef \emptyset$ $\forall a\in A$ showing $\SigmaP2$-hardness.
	
	For propositional \RC{}s, reduce from a $\QSAT$-instance $\Phi=\forall Y.\exists Z.\forall X.\bigvee_{c\in \mathcal C}c$ to an \RAF $\CF\eqdef(A,R,C)$, where $(A,R)$ is as above, but without the arguments~$\mathcal C$ and attacks involving~$\mathcal C$. 
	Then, define~$C(t)\eqdef \{\neg c \mid c\in\mathcal C\}$ and the complexity increases. 
	The construction works also for disj.\ \RC{}s (cf.\ Theorem~\ref{thm:hard}), rising the complexity to~$\Sigma_4^P$. %
	
	``Upper Bounds'': Given \RAF $\CF=(A,R,C)$, we are asked for an extension containing a specific argument $a\in A$.
	In essence, we need to follow the approach in the proof of Theorem~\ref{thm:membership}. 
	However, we need to know the elements in extension $E$ completely (opposed to before) and ensure $a\in E$.
	Thus, we must not only guess it, but also check respective maximality by a $\co\NP$-oracle call that increases the previous runtimes by one level.
\end{proof}

\subsection{Tight Complexity Bounds for Treewidth}\label{sec:fpt}
To avoid this high complexity, we consider treewidth, which has been used to efficiently solve argumenta\-tion~\cite{iccma}.
First, we define a graph representation.

\begin{definition}
Let~$\CF=(A,R,C)$ be an \RAF and let $(V_a, E_a)=\primal{C(a)}$ be given for every~$a\in A$. Then, the \emph{primal graph~$\primal{\CF}$} of~$\CF$
is defined by %
$\primal{\CF}\eqdef (\bigcup_{a\in A} V_a \cup A,$ $\bigcup_{a\in A} E_a \cup R \cup \{\{a,v\} \mid a\in A, c\in C(a), v\in\var(c)\}$. %
\end{definition}
Below, let $\CF=(A,R,C)$ be an \RAF and $\mathcal{T}{=}(T,\chi)$ be a TD of~$\primal{\CF}$. For any~$t$ of~$T$, $C_t\eqdef \{c\mid a\in A_t,\allowbreak c\in C(a), \var(c)\in\chi(t)\}$ denotes \emph{bag rejection conditions}.

\begin{example}\label{ex:graph-td}
Fig.~\ref{fig:graph-td} presents a graph representation and a TD for Ex.~\ref{ex:CCF-complex}.
Since the largest bag size is $4$, the treewidth of our TD is $3$. \qedexample
\end{example}

\begin{figure}
\centering
	\begin{tikzpicture}[arg/.style={circle,fill=black,inner sep=.75mm},y=.8cm]
		\node[arg, label={180:\bfseries T}] (T) at (-1,-1) {};
		\node[arg, label={\bfseries noS}] (C) at (0,0) {};
		\node[arg, label={0:\bfseries P}] (E) at (0,-1) {};
		\node[arg, label={0:\bfseries W}] (V) at (1,-1) {};
		\node[arg, label={0:\bfseries Te}] (S) at (1,-2.45) {};
		\node[arg, label={270:\bfseries Re}] (Re) at (0,-3.25) {};

		\node[arg, label={180: $p_{\text{dl}}$}] (restr) at (0,-1.75) {};
		\node[arg, label={180: $p_{\text{hw}}$}] (conv) at (0.,-2.45) {};
			\node[arg, label={280: $p_{\text{exp}}$}] (exp) at (1,-3.25) {};
				
		\foreach \f/\t in {conv/restr, conv/exp, T/restr, E/restr, S/exp, Re/exp, Re/conv, V/conv, V/restr, C/T,C/E,V/C,Re/S,S/V}{
			\path[-] (\f) edge (\t);
		}
\end{tikzpicture}
\hspace{0.5cm}
\begin{tikzpicture}[level distance= 2.5 em, sibling distance= 5.5 em,every node/.style={draw},
	edge from parent/.style={thin,-,black, draw},
	 leaf/.style = {draw, circle}]
\node  {\textcolor{black}{\bfseries noS, \bfseries P, $p_{d}$ }}
		child {node {\bfseries T, \bfseries noS, $p_{d}$}}
	child {node {\bfseries noS, \bfseries W, $p_{d}$} 
	child {node {\bfseries W, $p_{d},p_{h}$}		
	child {node {\bfseries W, $p_{h}$, \bfseries Te}
	child {node {\bfseries Te, $p_{h},p_e$, \bfseries Re}
}}}}
;
\end{tikzpicture}%
	\caption{Primal graph and a TD of Figure~\ref{fig:CCF-complex}.}\label{fig:graph-td}
      \end{figure}
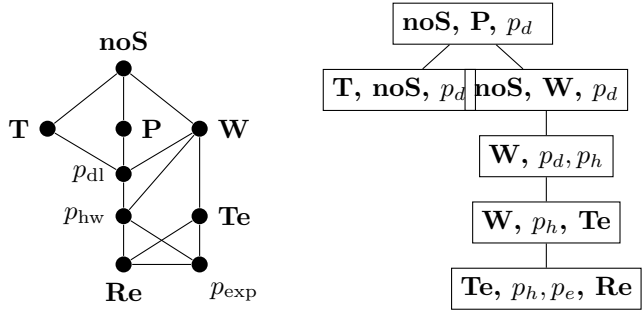
Next, we present runtime results in the form of upper bounds for the solution of our problems of interest when parameterized by treewidth.
To this end, we follow the idea of decomposition-guided (DG) reductions of related work~\cite{FichteEtAl21} to establish efficient algorithms for instances of bounded treewidth. 
Such DG reductions reduce polynomially from a source to a target, taking into account a TD~$\mathcal{T}$ of the source. 
By design, these reductions yield a TD of the target, suited for showing (tree)width dependencies.

\paragraph{Stable Semantics.} First, we compute stable extensions of~$(A,R)$ via a reduction to \BSAT.
To this end, we start with the following reduction, cf.~\cite{FichteEtAl21}.
This reduction uses a variable~$a$ for each argument~$a\in A$, which indicates whether~$a$ is in the extension or not.  %
To encode whether an argument is defeated up to a node~$t$, 
auxiliary variables of the form~$d^t_a$ for every node~$t$ of~$T$ and argument~$a\in A$ indicate whether~$a$ is attacked (``defeated'') by an argument~$b\in\chi(t)$ of the extension, resulting in
\emph{defeated variables}~$D\eqdef\{d^t_a\mid a\in A_t, t\text{ in }T\}$.

The DG reduction~$\mathcal{R}_{\Stab}(\CF,\mathcal{T})$ constructs a propositional formula: 
$	\mathcal{R}_{\Stab}(\CF,\mathcal{T}) \eqdef \exists A, D. \varphi_{\stab}(A,D)$, where $\varphi_{\stab}(A,D)$ is a CNF consisting of Formula~(\ref{stab:def})--(\ref{stab:remove}):
\begin{flalign}
	\label{stab:def}&d_a^t \rightarrow\bigvee_{\substack{t^*\in\children(t),a\in A_{t^*}}} d_a^{t^*}\, \vee \bigvee_{(b,a)\in R_t} b& \forall t\in T,a\in A_t\\
	\label{stab:conf}&\neg a \vee \neg b& \forall(a,b)\in R\\
	\label{stab:remove} & a \vee d_a^{\forg(a)}&  \forall a\in A
\end{flalign}
Intuitively, 
Formula~(\ref{stab:def}) define the defeated variables, (\ref{stab:conf}) ensure conflict-freeness and~(\ref{stab:remove})  take care that each argument is in the extension or defeated.

\paragraph{Simple \RC{}s.}
We take the DG reduction $\mathcal{R}_\stab$ and extend it by \RC{}s.
Since for an extension some clause needs to evaluate to false, we guide this evaluation along the TD~$\mathcal{T}=(T,\chi)$ of~$\primal{\CF}$.
We use auxiliary variables~$w^t$ for every node~$t$ of~$T$ and $w_c^t$ for~$t$ and every rejection condition~$c\in C_t$,
resulting in \emph{rejection witnesses}~$W\eqdef \{w^t, w_c^t \mid t\text{ of }T, c\in C_t\}$.
These auxiliary variables %
witness that some rejection condition in a node~$t$ is invalidated.
For simple \RC{}s, we define $\mathcal{R}_{\Stab,\simple}(\CF,\mathcal{T}) \eqdef 		\exists A, D, W. \varphi_{\stab}(A,D) \wedge \varphi_{\simple}(A,W)$, where
$\varphi_{\simple}(A,W)$ is in CNF consisting of Formula~(\ref{simpl:wit})--(\ref{simpl:remove2}).%
\begin{flalign}
  \label{simpl:wit}&w^t \rightarrow\bigvee_{\substack{t^*\in\children(t)}}  w^{t^*} \vee \bigvee_{c\in C_t} w_c^t& \forall t\in T \\
  \label{simpl:remove2} & w^{\rootOf({T})}\\
  \label{simpl:modl}&w_c^t \rightarrow \neg l& \forall t\in T, c\in C_t, l\in c\\%
  \label{simpl:moda}&w_c^t \rightarrow \bigvee_{\mathclap{a\in A_t: c\in C(a)}}a& \forall t\in T, c\in C_t
\end{flalign}

Formula~(\ref{simpl:wit}) take care that witnessing any invalidated rejection condition in a node~$t$ requires such a witness below~$t$ or some specific \RC~$c\in C_t$ for~$t$. Such a specific~$c$ requires at least some ``hosting'' argument~$a\in A_t$ to be in the extension, cf.\ Formula~(\ref{simpl:moda}), as well as corresponding truth values of literals, cf.\ Formula~(\ref{simpl:modl}).
Finally, we require that eventually~$w^{\rootOf({T})}$ holds, i.e., some \RC is invalidated as given by Formula~(\ref{simpl:remove2}).
With this reduction at hand, we establish:

\begin{theorem}\label{thm:simple}
Given any \RAF~$\CF=(A,R,C)$. Then, $\cons_{\stab}(\CF)$ for simple \RC{}s~$C$ can be solved in time~$\tower(1,\mathcal{O}(k))\cdot\poly(\Card{A})$, %
with treewidth~$k=\tw{\primal{\CF}}$. %
\end{theorem}
\begin{proof}[Proof (Idea)]
Take \emph{any} TD~$\mathcal{T}{=}(T,\chi)$ of~$\primal{\CF}$ and, without loss of generality, we may assume that for every~$t$ of~$T$, we have~$\Card{C_t}=\width(\mathcal{T})$, which can be obtained by splitting large~$C_t$ into smaller ones with the help of adding copy nodes to~$\mathcal{T}$.
Then, we construct $\mathcal{T}'{\,\eqdef\,}(T,\chi')$ with~$\chi'(t){\,\eqdef\,}\{d_a^t, w_c^t, w^t, w^{t^*} \mid t^*\in\children(t),a\in A_t, \in C_t\} \cup\chi(t)$ for every~$t$ of~$T$. Indeed, $\mathcal{T}'$ is a TD of~$\primal{\mathcal{R}_{\stab,\simple}(\CF,\mathcal{T})}$ with $\width(\mathcal{T}') \leq 3\cdot\width(\mathcal{T})+1$. Then, Proposition~\ref{prop:qbf} yields the result.
\end{proof}

While the cases above focused on the stable semantics, one might as well encode the \emph{computation of further semantics}, while reusing our encodings for the different fragments of rejection constraints above. For the decision problems, this leads to the results as given in the last column of Table~\ref{tbl:overviewresults}. 
For credulous reasoning, we obtain similar results.

\paragraph{Propositional \RC{}s.}
Consider the case for propositional \RC{}s allowing the use of additional variables~$B$ that are not among the arguments of the AF.
This additional expressiveness allows us to encode the problem to~$\TQSAT$ without additional auxiliary variables.

We define $\mathcal{R}_{\Stab,\prop}(\CF,\mathcal{T}) \eqdef $
		$\exists A, D. \forall B. \varphi_{\stab}(A,D) \wedge \varphi_{\prop}(A,B)$, %
where $B\eqdef \var(C)\setminus A$ are variables of~$C$ not in~$A$ and $\varphi_{\prop}(A,B)$ is in DNF consisting of Formula~(\ref{stab:mod2}). %
\begin{flalign}
	\label{stab:mod2}a \wedge  \bigwedge_{l\in c}\neg l\quad & & \forall a\in A, c\in C(a)  %
\end{flalign}
Intuitively, Formula~(\ref{stab:mod2}) ensure that for some argument of the extension, its \RC is invalid.
Then, we obtain:

\begin{theorem}\label{thm:prop}
Given any \RAF~$\CF=(A,R,C)$. Then, $\cons_{\stab}(\CF)$ for propositional \RC{}s~$C$ can be solved in time~$\tower(2,\mathcal{O}(k))\cdot\poly(\Card{A})$, %
with~$k=\tw{\primal{\CF}}$. %
\end{theorem}
\begin{proof}[Proof (Idea)]
Similar to the proof of Theorem~\ref{thm:simple}, we take \emph{any} TD~$\mathcal{T}{=}(T,\chi)$ of~$\primal{\CF}$
and construct $\mathcal{T}'{\,\eqdef\,}(T,\chi')$ with~$\chi'(t){\,\eqdef\,}\{d_a^t \mid a\in A_t\} \cup\chi(t)$ for every~$t$ of~$T$; Then, since~$\mathcal{T}'$ is a TD of~$\primal{\mathcal{R}_{\stab,\simple}(\CF,\mathcal{T})}$ and $\width(\mathcal{T}') \in\mathcal{O}(\width(\mathcal{T}))$, we conclude due to Proposition~\ref{prop:qbf}.
\end{proof}

\paragraph{Tight \RC{}s.}

For tight \RC{}s, additionally, we ensure justification of atoms.
In order to design a DG reduction,
we guide justifiability for an atom in a node~$t$ of~$T$ along the TD.
To this end, we rely on  %
\emph{justification variables}~$J\eqdef \{j_x^t, j_c^t\mid t\text{ of }T,x\in\chi(t)\setminus A, c\in C_t\}$.
These auxiliary variables indicate that a justification for an atom~$x$ or involving \RC~$c$ is found up to a node~$t$.
We define $\mathcal{R}_{\Stab,\tight}(\CF,\mathcal{T}) {\eqdef}
		\exists A, D. \forall B,J. \varphi_{\stab}(A,D) \wedge \varphi_{\tight}(A,B,J)$, where
$\varphi_{\tight}(A,B,J)$ is in DNF, which compri\-ses Formula~(\ref{stab:mod2})--(\ref{tight:remove3}).%
\begin{flalign}
  &j_x^{t} \wedge \bigwedge_{\substack{t^*\in\children(t),\\x\in \chi(t^*)}}\neg j_x^{t^*} \wedge\bigwedge_{c\in C_t, x\in H_c}\neg j_c^t \forall t\in T, x\in \chi(t)\setminus A_t\label{tight:prop}\\
&j_c^t \wedge l& \forall t\in T, c\in C_t, l\in B_c^- \cup \{\neg y\mid y\in B_c^+\}\;\label{stab:modasp}\\
\label{tight:def}&j_c^{t} \wedge\bigwedge_{a\in A_t: c\in C(a)}\neg a& \forall t\text{ of }T, c\in C_t\\
& x\wedge \neg j_x^{\forg(x)}&  \forall x\in \chi(t)\setminus A_t\;	\label{tight:remove3} %
\end{flalign}
Intuitively, Formula~(\ref{tight:prop}) ensure that if any of the universally quantified~$j_x^t$ claims to have found a justification for~$x$, this has to be indeed the case below~$t$ or involving an \RC~$c$. Such a justification involving an \RC is invalidated if any of the body atoms is not set accordingly, cf.~Formula~(\ref{stab:modasp}), or if it is not activated by any argument of the extension, cf.~Formula~(\ref{tight:def}). Finally, if none of those terms is satisfied, the actual goal is to find some atom that is not justified,
which is guaranteed by Formula~(\ref{tight:remove3}).

\begin{theorem}\label{thm:tight}
Given any \RAF~$\CF=(A,R,C)$. Then, $\cons_{\stab}(\CF)$ for tight \RC{}s~$C$ can be solved in time~$\tower(2,\mathcal{O}(k))\cdot\poly(\Card{A})$, %
with~$k=\tw{\primal{G}}$. %
\end{theorem}

Interestingly, the reduction can be extended even further to \emph{normal \RC{}s}, where there may be cyclic dependencies between atoms. Here one has to encode level mappings~\cite{Janhunen06} between atoms in a way that is treewidth-aware~\cite{\camera{Hecher20}HecherFandinno21}.

\paragraph{Disjunctive \RC{}s.}
For disjunctive \RC{}s, we have to care about subset-minimality with respect to the GL reduct. 
Thus, for a set of atoms to be an extension, (i) it must be an extension of the AF and (ii) the disjunctive \RC{}s must be invalidated, which is the case if either (ii.a) there is no model (cf.~$\varphi_{\prop}$) or (ii.b) there is a model of the GL reduct ($\varphi_{\textsf{red}}$) that is smaller ($\varphi_{\textsf{subs}}$).

Towards such a $3\hy\QSAT$ encoding, we require an auxiliary variable~$r$ (``reduct'') indicating that we aim for (ii.b) if~$r$  holds and for (ii.a) otherwise. Further, we need to guide the checking of subset-minimality along the TD by means of auxiliary variables~$S$, resulting in the \emph{subset-minimality variables}~$S\eqdef \{s^t, s_b^t\mid t\text{ of }T, b\in B\cap\chi(t)\}$.
We define
$\mathcal{R}_{\Stab,\disj}(\CF,\mathcal{T}) \eqdef \exists A, D. \forall B. \exists B',S,\{r\}.(\varphi_{\stab}(A,\allowbreak D) \wedge \varphi_{\textsf{red}}(B') \wedge \varphi_{\textsf{subs}}(B,B',S)) \wedge
	 (\varphi_{\prop}(A,B) \vee r)$, where formula $\varphi_{\textsf{red}}(B')$ is in CNF, comprising Formula~(\ref{red:mod2}):%
\begin{flalign}%
	&\neg r \vee\neg a \vee\bigvee_{x'\in H_c'}x' \vee \bigvee_{x'\in (B_c^+)'}\neg x' \vee \bigvee_{x\in B_c^-}x, %
        &\forall a\in A_t,c\in C_t\cap C(a)\label{red:mod2}
\end{flalign}
Formula~(\ref{red:mod2}) 
ensures that if~$r$ holds and an argument~$a$ is in the extension, then the GL reduct of its \RC{}s hold.
CNF $\varphi_{\textsf{subs}}(B,B',S)$ consists of Formula~(\ref{red:subset})--(\ref{red:removestricts}).
\begin{flalign}
	\label{red:subset} & \neg r \vee \neg b' \vee b& \forall b\in B\,\\
	\label{red:stricts}&\neg r \vee \neg {s}^t \vee\bigvee_{t^*\in\children(t)}{s}^{t^*} \vee \bigvee_{b\in B\cap\chi(t)}s_b^t& \forall t\in T\,\\
	&\neg {s}_b^t \vee b & \forall t\in T, \label{red:stricts2}b\in B\cap\chi(t)\, \\
	&\neg {s}_b^t \vee \neg b' & \forall t\in T,\label{red:stricts3}b\in B\cap\chi(t)\, \\
	\label{red:removestricts}&s^{\rootOf(T)}
\end{flalign}	
Intuitively, $\varphi_{\textsf{subs}}$ ensures that the assignment over~$B'$ is (strictly) subset-smaller than the one over~$B$. The non-strict subset-inclusion is modeled by Formula~(\ref{red:subset}). Then, Formula~(\ref{red:stricts}) take care that if strict inclusion is encountered up to~$t$, this stems from below~$t$ or for a specific~$b\in B$,
which is defined by Formula~(\ref{red:stricts2}) and~(\ref{red:stricts3}).
Finally, such a strict inclusion has to be encountered, cf.~Formula~(\ref{red:removestricts}).
Then, this reduction yields the following result.
\begin{theorem}\label{thm:disj-tw}
  Given any \RAF $\CF = (A, R, C)$.  Then, $\cons_{\stab}(\CF)$ for
  disj.\ \RC{}s can be solved in
  time~$\tower(3,
  {O}(\tw{\primal{G}}))\cdot\poly(\Card{A})$. %
\end{theorem}

While the cases above focus on the stable semantics, one could also encode the \emph{computation of further semantics} while reusing our encodings for different fragments of \RC{}s above. For decision problems, this leads to the results %
in the last column of Table~\ref{tbl:overviewresults}
and similar for credulous reasoning.
\begin{theorem}\label{thm:disj-tw-ub}
  For any \RAF~$\CF=(A,R,C)$ the problems $\cred_\semi$ and~$\cred_\stag$ for simple, propositio\-nal/tight, and disjunctive \RC{}s~$C$ can be solved in~$\tower(i, \mathcal{O}(\tw{\primal{\CF}}))\cdot\poly(\Card{A})$, where $i=2$, $i=3$, and $i=4$, respectively.
\end{theorem}~\\[-2.75em]
\begin{proof}[Proof (Idea)]
  We encode the computation of semi-stable and stage extensions
  similar to Fichte~et~al.~\cite{FichteEtAl21}, which linearly preserves treewidth.
  To this we add the encoding of handling simple, propositional/tight, and disjunctive rejection constraints, as above.
\end{proof}

\subsection{Matching Lower Bounds under ETH}\label{sec:lb}

Unfortunately, the above runtime results can not be significantly
improved, which we establish by focusing on lower bounds.  
The reductions above preserve treewidth linearly.

\begin{theorem}\label{thm:preserve}
The reductions in the proof of Theorem~\ref{thm:hard} %
for \RC{}s linearly preserve the treewidth. %
\end{theorem}~\\[-2.75em]%
\begin{proof}
We show this result by taking a TD~$\mathcal{T}=(T,\chi)$ of the primal graph~$\primal{\varphi}$ of the corresponding matrix (formula)~$\varphi$.
From this we define a TD~$\mathcal{T}'=(T,\chi')$ of~$\primal{\CF}$ of the constructed \RAF~$\CF$.

``$\props$ \RC{}s'': %
For every node~$t$ of~$T$, 
we define~$\chi'(t)\eqdef\chi(t) \cup \{x'\mid x\in \chi(t)\}$. 
Indeed, $\mathcal{T}'$ is a TD of~$\primal{D}$ and~$\Card{\chi'(t)}\leq 2\cdot\Card{\chi(t)}$ for every node~$t$ of~$T$.

``$\simple$ \RC{}s'': %
For every node~$t$ of~$T$, 
we define~$\chi'(t)\eqdef\chi(t) \cup \{x'\mid x\in \chi(t)\} \cup \{v_c \mid c\in \varphi\}$. 
Indeed, $\mathcal{T}'$ is a TD of~$\primal{\CF}$. 
Observe that there might be bags of nodes containing more than~$\mathcal{O}(\width(\mathcal{T}))$ many elements of the form~$v_c$.
However, these arguments~$v_c$ do not have to occur in those bags together.
So, one can slightly modify~$\mathcal{T}'$ by copying such nodes in the tree, such that each copy only gets one of these variables~$v_c$. We refer to the resulting TD~$\mathcal{T}''=(T'',\chi'')$ of~$\primal{\CF}$ 
Then, we have that~$\width(\mathcal{T}'') \leq 2\cdot\width(\mathcal{T}) + 1$ for every node~$t$ of~$T$.
``$\tight$ \RC{}s'': %
The proof proceeds similar to the case of~$\props$ \RC{}s.
\end{proof}

\begin{theorem}\label{thm:tw-lb}
  Given any \RAF~$\CF=(A,R,C)$. 
  Under ETH, $\cons_\conf$ for simple ($i=1$), propositio\-nal/tight ($i=2$), and disjunctive ($i=3$) \RC{}s cannot be decided in time
  $\tower(i, {o}(\tw{\primal{\CF}}))\cdot\poly(\Card{A})$; %
  $\cred_\semi$ and~$\cred_\stag$ cannot be decided in time
  $\tower(i+1, {o}(\tw{\primal{\CF}}))\cdot\poly(\Card{A})$.
\end{theorem}
\begin{proof}
  Assume towards a contradiction that this is not the case for some~$i$, i.e.,
  we can decide~$\conf$ with the desired \RC{}s in time~$\tower(i, {o}(\tw{\primal{\CF}}))\cdot\poly(\Card{A})$.
  Recall that the reductions of Theorem~\ref{thm:hard} are correct and linearly preserve %
  treewidth by Theorem~\ref{thm:preserve}. Consequently, %
  we reduce from~$i\hy\QSAT$, contradicting %
  Proposition~\ref{qbf:lb}. 
  \end{proof}
  
  In the following, we prove that the reduction~$\mathcal{R}_{\stab,\disj}$ is indeed correct.
  The correctness for the other (simpler) cases is established similarly.
  
  \begin{lemma}[Correctness]\label{thm:corr-disj}
  Given any \RAF~$\CF=(A,R,C)$ and a TD~$\mathcal{T}=(T,\chi)$ of~$\primal{\CF}$. Then, there is a stable extension of~$\CF$ if and only if~$\mathcal{R}_{\stab,\disj}(\CF,\mathcal{T})$ is satisfiable. 
  \end{lemma}
  \begin{proof}[Proof (Sketch)]~\\
  Recall the defined reduction Formulas~(\ref{red:mod2})--(\ref{red:removestricts}).
  
    ``$\Rightarrow$'': Let $E\subseteq A$ be a stable extension of $\CF$.
    Then $E$ is an extension of $(A,R)$.
    Moreover, either $\mathcal P =C(E) \cup E$ %
    is inconsistent, or for every model $M$ of $\mathcal P$, there is a model $M'\subset M$ for $\mathcal P ^M$.
    Clearly, the presence of such an $E$ is also realized by the TD $\mathcal T$ of $\mathcal G_{\CF}$, ensuring that there are sets $A$ and $D$ such that $\exists A,D .\varphi_{\stab}(A,D)$ is true.
    Since $\mathcal P$ is inconsistent then one of the following two cases hold for each set $M$ of atoms.
    (1) There is some rule $p:= H(p) \leftarrow B^+(p), B^-(p)$ such that $(H(p) \cup B^-(p))\cap M = \emptyset$ and $B^+(p) \subseteq M$.
    (2) There is a set $M'\subset M$, such that $M'$ is a model of $\mathcal P^M$.
    That is, for each $M$ either $M\not\models\mathcal P$ or there is an $M'\subset M$ such that $M'\models\mathcal P^M$.
    Moreover, the witness in each case (for contradiction in the case of $M\not\models\mathcal P$, and for satisfaction/subset minimality in the case of $M'\models\mathcal P^M$) is guided along the TD, proving that $\mathcal{R}_{\Stab,\disj}(\CF,\mathcal{T})$ is true.
  
    ``$\Leftarrow$'': Suppose that $\mathcal{R}_{\Stab,\disj}(\CF,\mathcal{T})$ is true. 
    Let $E$ be a set of arguments corresponding to variables (set to true) in $A$.
    The definition of $\varphi_{\stab}(A,D)$ (cf. (\ref{stab:def})--(\ref{stab:remove})) ensures that $E$ is a stable extension of $(A,R)$.
    Then the set $B$ (quantified as $\forall B$) either satisfies $\varphi_{\prop}(A,B)$, (i.e., it does not satisfy $\mathcal P$) or it satisfies $r$ (activates the reduct variable).
    This together with the guessed set $B'$ satisfies that $B'\subset B$ ((\ref{red:subset})--(\ref{red:stricts3})) and for every $a\in A$ (if $a$ is in the extension), the  variables $x,x'$ (simulating the program) are selected in such a way that the corresponding program is satisfied by the atoms chosen by $B'$ (\ref{red:mod2}).
    This ensures that $\mathcal P$ is not satisfiable since sets $B$ and $B'$ are nothing but variables corresponding to set of atoms simulating $M$ and $M'$.
    This completes the proof.
  \end{proof}

\section{Conclusion and Future Work}
We study flexible rejecting conditions (RCs) in abstract
argumentation, which provide more fine-grained conditions to commonly
studied rejection such as maximality.
It allows for scenarios where we are interested in not accepting
certain arguments in combination with others.
We establish under which situations rejection is harder than its
counterpart of acceptance conditions.
We provide valuable insights into the differences and interactions between rejection and acceptance in general, which by itself improves the understanding of abstract argumentation.
In fact, we present a detailed analysis of expressiveness and
complexity when \RC{}s consist of propositional formulas or ASP
programs.  Table~\ref{tbl:overviewresults} gives an overview of our
complexity results.  Our results show that rejection gives rise to
natural problems for the third and fourth levels of the polynomial
hierarchy.

It is notable that the higher expressivity and complexity for RCFs can be attributed to the shift from ``acceptance'' to ``rejection''.
Furthermore, although our results on simulating CAFs with RCFs utilize additional variables, this generalization is not the primary source of complexity increase.
This can be observed by noting that extending CAFs with auxiliary variables does not alter their reasoning complexity. 
The high-level idea is that CAFs model a formalism according to ``Exists Extension, Exists Assignment'', whereas our concept based on rejection follows the pattern ``Exists Extensions, Forall Assignments'', which makes RCFs more expressive (assuming $\Ptime \neq \NP$).

Investigations of \RAF{}s with additional semantics and \RC{}s are interesting for future work. 
For example, for \RC{}s with nonground programs of fixed arity, we expect complexity results up to the fifth level of the polynomial hierarchy.
We are also interested in implementing \RAF{}s and comparing their performance with AF solvers~\cite{iccma}. 
There are also meta-frameworks for treewidth, e.g.~\cite{CharwatWoltran19,HecherThierWoltran20}, which are relevant for the implementation of \RAF{}s using treewidth.

Improving the modeling with a stronger language (ASP) is not only natural, but also necessary for certain properties, based on standard assumptions in complexity theory.
While we understand that some may prefer modeling in classical logic over ASP, we find it difficult to follow the reasoning. 
We respect the preference for classical or low-level propositional logic and believe that higher languages should also be considered. 
This also opens up the possibility to consider other natural extension of ASP, such as first-order logic.

\paragraph*{Acknowledgments.}
Authors are ordered alphabetically.
The work has been carried out while Hecher visited the Simons
Institute at UC Berkeley.
Research is supported by 
the Austrian Academy of Sciences (\"OAW), DOC Fellowship;
the Austrian Science Fund (FWF), grants P30168 and J4656;
the Society for Research Funding in Lower Austria (GFF, Gesellschaft f\"ur Forschungsf\"orderung N\"O) grant ExzF-0004; the Vienna Science and Technology Fund (WWTF) grant ICT19-065;
ELLIIT funded by the Swedish government;
the Deutsche
Forschungsgemeinschaft (DFG, German Research Foundation), grants TRR 318/1 2021 – 438445824 and ME 4279/3-1 (511769688); the European Union’s Horizon Europe research and innovation programme within project ENEXA
(101070305); and the Ministry of Culture and Science of North Rhine-Westphalia (MKW NRW) within project SAIL, grant NW21-059D.

\bibliographystyle{plain}
\bibliography{references}

\end{document}